\documentclass[preprint,12pt]{colt2026} 

\usepackage[capitalize,noabbrev]{cleveref}
\usepackage{enumitem}

\makeatletter
\@ifundefined{assumption}{\newtheorem{assumption}{Assumption}}{}
\makeatother

\crefname{assumption}{Assumption}{Assumptions}
\Crefname{assumption}{Assumption}{Assumptions}

\title[Semantic Reference Frame for Residual-Stream Dynamics]{SemRF: A Semantic Reference Frame for Residual-Stream Dynamics in Language Models}

\usepackage{times}
\usepackage{xurl}
\usepackage{microtype}

\newcommand{\R}{\mathbb{R}}
\newcommand{\ip}[2]{\langle #1,#2\rangle}
\newcommand{\norm}[1]{\lVert #1\rVert}

\coltauthor{%
 \Name{Jian Gu} \Email{jian.gu@monash.edu}\\
 \addr Monash University
 \AND
 \Name{Aldeida Aleti} \Email{aldeida.aleti@monash.edu}\\
 \addr Monash University
 \AND
 \Name{Chunyang Chen} \Email{chun-yang.chen@tum.de}\\
 \addr Technical University of Munich
 \AND
 \Name{Hongyu Zhang} \Email{hyzhang@cqu.edu.cn}\\
 \addr Chongqing University
}

\begin{document}


\maketitle

\begin{abstract}
Residual-stream analysis asks how a language model's computation evolves across depth, but this question requires a stable semantic coordinate system. Intermediate decoding is meaningful only when readout coordinates remain comparable across layers. In language models, this comparability is mediated by the token interface: if embedding-side anchors and unembedding-side readout disagree on the semantic span under analysis, apparent hidden-state motion may reflect measurement drift rather than computation.
We introduce \emph{Semantic Reference Frames} (SemRF), an anchor-based formalism that separates semantic measurement from residual-stream dynamics. A SemRF fixes the anchors used in the analysis and measures residual states against them. Pseudo-inverse tying supplies an exact anchor-synchronization case; under restricted bi-invertibility on the chosen span, SemRF gives stable semantic-basis coordinates, distortion bounds, and controlled near-identity frame changes.
With the frame fixed, residual computation becomes a depthwise semantic trajectory. The same anchors induce a semantic Voronoi diagram: semantic distance, equivalently semantic evidence such as logits, assigns each layer to a coarse cell, while coordinates retain within-cell motion and margins. We define layerwise steps, target-direction contribution profiles, and imbalance diagnostics, then use the observed Voronoi trace to define a margin-relaxed semantic tube. The canonical trace is the minimum-action path inside this tube. In a nonempty tube with at least one positive quadratic weight, it is unique and satisfies a discrete spline equation away from active constraints. Excess action controls step mismatch, curvature mismatch, interior deviation, and target-direction contribution-profile mismatch.
Finally, low curvature energy implies piecewise-linear compressibility. This yields a local notion of knowledge density: within a fixed admissible frame and tolerance, lower trace complexity means that the depth-by-coordinate trajectory can be summarized by fewer semantic knots. Through the residual network's parameter-to-trajectory map, this also gives a conditional link to parameter efficiency: among admissible parameter settings that fit the same data, lower-action and lower-complexity traces use fewer effective semantic degrees of freedom. The guarantees require controlled interface error, small projection residual, and explicit tube constraints.
\end{abstract}


\section{Introduction}

Transformer language models update a shared residual state through layerwise writes. This suggests a trajectory view: each layer moves the same state, and intermediate residual vectors carry partial semantic evidence. LogitLens and Tuned Lens decode hidden states into vocabulary space across depth~\cite{nostalgebraist2020logitlens,belrose2023eliciting}, but such trajectories are meaningful only when readout coordinates remain comparable across layers. If embedding-side anchors and unembedding-side readout disagree on the semantic span under analysis, apparent semantic motion may reflect measurement drift rather than computation.

We introduce \emph{Semantic Reference Frames} (SemRF), an anchor-based formalism for separating measurement from dynamics. A SemRF fixes the semantic anchors used for the analysis and measures residual states against them. The main structural assumption is restricted bi-invertibility of the token interface on this anchor span. Pseudo-inverse tying supplies an exact anchor-synchronization case; when the restricted interface condition holds, semantic coordinates are stably measurable and interface mismatch yields explicit distortion bounds. SemRF specifies when semantic readout, diagnosis, and comparison are well posed.

Once measurement is controlled, residual evolution becomes a semantic trajectory. The same semantic anchors induce a semantic Voronoi diagram: each state is assigned to the anchor with smallest semantic distance, equivalently largest semantic evidence such as a logit. This is the sense in which we use ``Voronoi'' throughout; ordinary Euclidean nearest-neighbor Voronoi is only the special case induced by a Euclidean distance score. The diagram gives cells, faces, margins, and a coarse Voronoi trace; continuous coordinates record within-cell motion, cell crossings, and reversals along a fixed target direction. We combine the two views through a margin-relaxed \emph{semantic tube}: the observed Voronoi trace defines a coarse route, and a path action selects the minimum-action path inside the tube as the canonical trace. The terminology is close to Semantic Tube Prediction~\cite{huang2026semantic}, but here the tube is an analysis constraint induced by measured Voronoi cells rather than a training objective or geodesic prior. When tube constraints are active, the comparison is no longer just endpoint interpolation.

The scope of our research is local and diagnostic: after choosing semantic anchors, checking interface and projection errors, and fixing tube constraints, depthwise semantic computation becomes a measurement-and-trajectory problem. Within this setting, we define layerwise semantic steps, target-direction contribution profiles, and imbalance measures; prove stability under nearby admissible frames; and show that excess action controls deviations from the canonical trace. Low-curvature traces are piecewise-linearly compressible, yielding a local notion of knowledge density: for a fixed admissible frame and tolerance, lower trace complexity means that the depth trajectory can be summarized by a few semantic knots rather than by all layer-by-coordinate states. Parameter-efficiency claims are read through the induced trajectory map \(\theta\mapsto z_{0:L}(x;\theta)\), not through a direct identification between hidden-state locations and unique parameter values: after interface alignment stabilizes the semantic basis, parameter settings can be compared by the complexity of the semantic transport they realize on the same data.
Our contributions are as follows:

\begin{enumerate}
\item \emph{Controlled semantic measurement.} We formalize semantic reference frames as anchor-based coordinate systems for residual-stream analysis, identify restricted interface conditions under which semantic coordinates are stably measurable, and derive distortion bounds on the chosen anchor span.

\item \emph{Stable semantic dynamics and canonical traces.} We define layerwise steps, target-direction contribution profiles, and imbalance diagnostics; prove cross-frame stability; and use an anchor-induced semantic Voronoi diagram to define a margin-relaxed semantic tube whose minimum-action path is the canonical trace.

\item \emph{Deviation control and local knowledge density.} We show that excess action controls specified deviations from the canonical trace, and that low curvature energy yields piecewise-linear compressibility, giving a frame-relative interpretation of knowledge density and a conditional bridge to parameter efficiency.
\end{enumerate}

\section{Related Work}

We position SemRF relative to work on intermediate readout, LM semantics, residual-stream analysis, geometric control, and efficiency.

\paragraph{Intermediate Vocabulary Readout.}
LogitLens and Tuned Lens recover token-level information from intermediate residual states~\cite{nostalgebraist2020logitlens,belrose2023eliciting}. Probing, contextual-geometry, and latent-visualization studies examine how semantic content is organized in hidden representations~\cite{petroni2019language,reif2019visualizing,ethayarajh2019contextual}. SemRF isolates the measurement issue behind these methods: trajectory-level claims require coordinates that remain comparable across depth, not a readout recalibrated independently at each layer.

\paragraph{Semantics and Interface Alignment.}
Recent work formulates LM semantics through vocabulary-defined semantics, semantic transition analysis, latent semantic alignment, and cross-model semantic transfer~\cite{gu2024vocabulary,gu2025semantic,gu2026beyond}. Pseudo-inverse tying connects semantic stability to synchronization between the input embedding and output unembedding on the relevant span~\cite{gu2025rethinking}. SemRF builds on this line by making the semantic anchors explicit, using them to form anchor-induced Voronoi cells, and turning restricted bi-invertibility and admissibility into quantitative conditions for comparing anchor-restricted coordinates across depth.

\paragraph{Residual-Stream Analysis.}
Mechanistic interpretability views Transformer computation as successive writes to a shared residual stream~\cite{elhage2021mathematical,olsson2022context,elhage2021circuits}. This perspective supports circuit analyses, causal mechanism identification, sparse autoencoders, superposition models, and residual-stream factorization~\cite{conmy2023towards,wu2023interpretability,huben2023sparse,bricken2023monosemanticity,templeton2024scaling,lawson2024residual,elhage2022toy}. Related work also studies representation stabilization and information transport through depth~\cite{yu2023exploring,janiak2024characterizing,shai2024transformers,meng2022locating}. SemRF asks when depth-indexed semantic motion reflects computation rather than readout drift.

\paragraph{Representation Geometry and Control.}
Work on linear features, representation steering, internal-update equivalence, isotropy, and embedding geometry studies when representations admit stable geometric interpretation~\cite{park2023linear,zou2023representation,goldwaser2025equivalence,hu2026representational,liang2021learning,ethayarajh2019contextual}. Semantic Tube Prediction uses a geometric trajectory prior as a JEPA-style learning signal~\cite{huang2026semantic}. SemRF uses the related idea as an analysis device: its semantic tube is the margin-relaxed set of paths that follow the trusted Voronoi trace induced by semantic anchors in a validated frame.

\paragraph{Parameter Efficiency and Knowledge Density.}
Scaling laws and capability-density analyses study efficiency at the level of parameters, data, and compute~\cite{kaplan2020scaling,hoffmann2022training,xiao2025densing}. SemRF uses a narrower notion: after admissible measurement is fixed, parameter settings that solve the same data can be compared by the action and complexity of the semantic traces they induce. This trace-level ranking is complementary to global efficiency metrics, not a replacement for them.

\section{Coordinates and Measurement}
\label{sec:setup}

We begin with the measurement problem. Let $h_{\ell,t}\in\R^d$ denote the residual-stream state at depth $\ell\in\{0,\dots,L\}$ and position $t$, and write
\[
h_{\ell+1,t}=h_{\ell,t}+u_{\ell,t}(h_{\ell,t}),\qquad \Delta h_{\ell,t}:=u_{\ell,t}(h_{\ell,t}).
\]
We analyze a fixed position $t$ and suppress it. Depth $\ell$ is discrete time. A \emph{semantic reference frame} (SemRF) consists of semantic anchors $\mathcal B$ and a coordinate map $\phi:\R^d\to\R^K$. It induces coordinates $z_\ell:=\phi(h_\ell)$ and increments $\Delta z_\ell:=z_{\ell+1}-z_\ell$. This section asks when these coordinates are stable measurements of the chosen anchor span.

\subsection{Coordinates and Vocabulary Readout}
A SemRF starts by fixing the semantic anchors that define the analysis. Let $\mathcal{B}=\{b_k\}_{k=1}^K\subset \R^d$, $K\ge2$, be these anchors, and let
\begin{equation}
z = \phi(h)\in\R^K,
\label{eq:phi_def}
\end{equation}
be the corresponding semantic coordinates of a residual state $h\in\R^d$.

\begin{definition}[SemRF coordinates]
A semantic coordinate system is a pair $(\phi,\mathcal{B})$ consisting of anchors $\mathcal{B}$ and a coordinate map $\phi:\R^d\to\R^K$. The coordinate $z_i=\phi_i(h)$ is interpreted as the semantic evidence of state $h$ for anchor $b_i$.
\end{definition}

\paragraph{Semantic-basis coordinates.}
Throughout the paper, $z=\phi(h)$ denotes the coordinate vector in the semantic basis fixed by $\mathcal B$. Margins, Voronoi assignments, steps, action values, and compressibility quantities are all computed in this single reference frame.

\paragraph{Notation.}
$\ip{\cdot}{\cdot}$ is the standard Euclidean inner product in the relevant ambient coordinate space, and $\norm{\cdot}$ denotes the induced $\ell_2$ norm on vectors and the corresponding operator norm on matrices.

\paragraph{Vocabulary anchors and interface readout.}
Vocabulary readout is the motivating example. Let $W_{\mathrm{in}}\in\R^{V\times d}$ be the input embedding and $W_{\mathrm{U}}\in\R^{V\times d}$ the unembedding. Writing $E:=W_{\mathrm{in}}^\top\in\R^{d\times V}$, the $i$-th column $e_i$ is the input-side anchor for token $i$. Output-side anchors can be obtained from $\tilde B:=W_{\mathrm{U}}^{+}\in\R^{d\times V}$. A hidden state is measured by its evidence for the chosen anchors, for example through logit coordinates or cosine similarity~\cite{gu2024vocabulary,gu2025semantic}.

\paragraph{Semantic Voronoi cells and margins.}
The same semantic anchors in $\mathcal B$ induce a semantic Voronoi diagram through semantic distances, equivalently through monotone semantic evidence scores. In the logit case, larger logit evidence means smaller semantic distance; for cell assignment one may set $d_i(h):=-\phi_i(h)$. Thus the cell of anchor $b_i$ is
\[
\mathcal V_i:=\{h:\phi_i(h)\ge \phi_j(h)\ \text{for all }j\}.
\]
The induced Voronoi index is
\[
\kappa(h):=\arg\max_i \phi_i(h),
\]
with ties broken by a fixed rule. A Voronoi face is where two anchors tie, $\phi_i(h)=\phi_j(h)$. The margin
\[
\operatorname{mar}(h):=\phi_{(1)}(h)-\phi_{(2)}(h),
\]
where $\phi_{(1)}(h)$ and $\phi_{(2)}(h)$ are the largest and second-largest coordinates, is the semantic-distance gap to the nearest competing face. Thus $\kappa(h)$ gives a coarse semantic state, while $\phi(h)$ and its margin retain within-cell information. Throughout the paper, ``Voronoi'' refers to this semantic-distance/evidence diagram. It reduces to ordinary Euclidean Voronoi only for special distance-derived scores; the results below use only the coordinate dominance inequalities $z_i\ge z_j$, which are convex halfspace constraints.

\subsection{Interface Alignment and Admissibility}
We now state when vocabulary readout becomes a controlled measurement procedure on the chosen anchor span. The first two propositions record exact and approximate synchronization cases; the assumption that follows is the local condition used by the later bounds.

\begin{proposition}[Exact interface synchronization]
\label{prop:pitz_canonical}
Suppose there exist matrices $Z\in\R^{V\times d}$ and an invertible $T\in\R^{d\times d}$ such that (i) the embedding and unembedding are parameterized as $W_{\mathrm{in}}=ZT^{-1}\in\R^{V\times d}$ and $W_{\mathrm{U}}=ZT^\top\in\R^{V\times d}$, and (ii) $Z$ has orthonormal columns: $Z^\top Z=I_d$.
Then the unembedding pseudoinverse recovers the input-side anchor matrix:
\[
W_{\mathrm{U}}^{+} = T^{-\top}Z^\top = W_{\mathrm{in}}^\top.
\]
Equivalently, the output-side pseudoinverse anchors coincide with the embedding anchors, so semantic coordinates defined via the restricted unembedding pseudoinverse are consistent across the input and output interfaces.
\end{proposition}

\begin{proposition}[Approximate interface synchronization]
\label{prop:approx_pit}
Let $E_{\mathcal B}\in\R^{d\times K}$ have full column rank and write $E_{\mathcal B}^+=(E_{\mathcal B}^\top E_{\mathcal B})^{-1}E_{\mathcal B}^\top$. Let $\eta:=\|U_{\mathcal B}^\top-E_{\mathcal B}^+\|$ measure the discrepancy between a restricted readout and the anchor pseudoinverse. Then
\[
\|U_{\mathcal B}^\top E_{\mathcal B}-I_K\|\le \eta\,\|E_{\mathcal B}\|.
\]
Moreover, if $U_{\mathcal B}^\top=\tilde E_{\mathcal B}^+$ for a nearby full-column-rank synthesis matrix $\tilde E_{\mathcal B}$ with $\|E_{\mathcal B}-\tilde E_{\mathcal B}\|\le \delta$, then
\[
\|U_{\mathcal B}^\top E_{\mathcal B}-I_K\|\le \|\tilde E_{\mathcal B}^+\|\,\delta.
\]
\end{proposition}

\begin{assumption}[Token-interface bi-invertibility]
\label{ass:interface}
Let $\mathcal{B}=\{b_k\}_{k=1}^K$ be the anchor set used by $\phi$.
There exist an anchor synthesis matrix $E_{\mathcal{B}}\in\R^{d\times K}$ and a restricted readout matrix $U_{\mathcal{B}}\in\R^{d\times K}$ such that
\[
\norm{U_{\mathcal{B}}^\top E_{\mathcal{B}} - I_K} \le \varepsilon,
\]
with $E_{\mathcal{B}}$ full column rank and $\|E_{\mathcal{B}}^+\|\le C_E$, $\|U_{\mathcal{B}}\|\le C_U$ for fixed frame constants $C_E,C_U$.
\end{assumption}

This local assumption rules out spurious coordinate drift created by the token interface. It requires stable synthesis and readout on the chosen anchor span, not a global semantic basis.

\begin{definition}[Admissible SemRF regime]
Fix tolerances $\eta_{\mathrm{int}},\eta_{\mathrm{proj}}>0$. We call a prediction event $(\theta,x,y)$ \emph{admissible} for a frame $(\phi,\mathcal B)$ when the interface mismatch satisfies
\[
\norm{U_{\mathcal B}^\top E_{\mathcal B}-I_K}\le \eta_{\mathrm{int}},
\]
and the relative projection residual of each analyzed state satisfies
\[
\frac{\|h_\ell-E_{\mathcal B}E_{\mathcal B}^{+}h_\ell\|}{\|h_\ell\|}\le \eta_{\mathrm{proj}},\qquad \ell=0,\dots,L.
\]
The same requirement may also be imposed on residual writes $\Delta h_\ell$ when step-level conclusions are needed.
\end{definition}

All quantitative results below are conditional on admissibility. Large interface mismatch means unstable readout; large projection residual means the chosen anchors miss active computation. Target-direction statements additionally require a fixed unit direction, either the target-token coordinate axis or a unit contrast direction.

\subsection{Diagnostics and Distortion Bounds}
The same diagnostics control how far measured coordinates can deviate from intended anchor-span coordinates.

\paragraph{Diagnostics.}
The key observables are the interface mismatch $\norm{U_{\mathcal B}^\top E_{\mathcal B}-I_K}$ and the relative projection residual
\[
r(h):=h-E_{\mathcal B}E_{\mathcal B}^+h,
\qquad \frac{\|r(h)\|}{\|h\|}.
\]
The first tests stability of synthesis and readout on the anchor span; the second tests whether the analyzed state lies mostly inside that span. Anchor choice therefore determines the semantic alternatives, interface conditioning, and reliability of later trajectory bounds.

\begin{proposition}[Anchor-span interface distortion]
\label{prop:coord_distortion}
Assume \cref{ass:interface} and use the restricted linear readout coordinates $\phi_{\mathrm{rd}}(h)=U_{\mathcal{B}}^\top h$.
For any coordinate vector $c\in\R^K$, let $h_c:=E_{\mathcal{B}}c$ be its recomposition into the anchor subspace.
Then
\[
\norm{\phi_{\mathrm{rd}}(h_c)-c}=\norm{(U_{\mathcal{B}}^\top E_{\mathcal{B}}-I_K)c}\le \varepsilon\,\norm{c}.
\]
More generally, for any decomposition $h=E_{\mathcal{B}}c+r$,
\[
\norm{\phi_{\mathrm{rd}}(h)-c}\le \varepsilon\,\norm{c}+\norm{U_{\mathcal{B}}}\,\norm{r}.
\]
\end{proposition}

\begin{corollary}[Trajectory distortion under bounded drift]
\label{cor:trajectory_distortion}
Let $\{h_\ell\}_{\ell=0}^L$ be any residual trajectory and suppose each state decomposes as $h_\ell=E_{\mathcal{B}}c_\ell+r_\ell$.
Let $z_\ell=\phi_{\mathrm{rd}}(h_\ell)$ and $\Delta c_\ell:=c_{\ell+1}-c_\ell$.
Then
\[
\norm{z_\ell-c_\ell}\le \varepsilon\norm{c_\ell}+\norm{U_{\mathcal{B}}}\norm{r_\ell},
\qquad
\norm{\Delta z_\ell-\Delta c_\ell}\le \varepsilon(\norm{c_{\ell+1}}+\norm{c_\ell})+\norm{U_{\mathcal{B}}}(\norm{r_{\ell+1}}+\norm{r_\ell}).
\]
\end{corollary}

These bounds identify the regime in which measured semantic motion is attributable to the residual trajectory rather than to interface error or off-span residuals. Section~\ref{sec:dynamics} then studies which dynamical features remain stable under admissible changes of frame.

\section{Layerwise Semantic Dynamics}
\label{sec:dynamics}
With a checked SemRF, residual-stream computation induces a depthwise trajectory in the semantic-basis coordinates fixed in Section~\ref{sec:setup}. For a prediction event $(\theta,x,y)$, write $z_{0:L}$ for this trajectory. This section defines step observables, proves their stability under admissible frame changes, and introduces the tube-constrained action used in Section~\ref{sec:results}.

\subsection{Step Structure and Path Imbalance}
The semantic anchors induce Voronoi cells, so each layer has a coarse state $\kappa_\ell:=\kappa(h_\ell)$ and margin $\operatorname{mar}(h_\ell)$ to the nearest Voronoi face. The coordinates $z_\ell$ retain the within-cell motion and the evidence changes that drive cell crossings. For a trajectory $\{h_\ell\}_{\ell=0}^L$, define the induced semantic field at layer $\ell$ by
\begin{equation}
F_\ell(h):=\phi\!\bigl(h+u_\ell(h)\bigr)-\phi(h)\in\R^K.
\label{eq:field_def}
\end{equation}
Along the realized trajectory, $F_\ell(h_\ell)=\Delta z_\ell$.

\begin{assumption}[Local linearization]
\label{ass:linearize}
Along trajectories of interest, $\phi$ admits a first-order approximation with bounded Jacobian $J_\phi$:
$\phi(h+\delta)=\phi(h)+J_\phi(h)\delta + r(h,\delta)$ with $\norm{r(h,\delta)}\le C_r\norm{\delta}^2$.
\end{assumption}

\begin{proposition}[Induced-step linearization]
\label{prop:first_order}
Under \cref{ass:linearize}, the induced semantic step satisfies
\[
\Delta z_\ell = J_\phi(h_\ell)\,\Delta h_\ell + \varepsilon_\ell,
\qquad \norm{\varepsilon_\ell}\le C_r \norm{\Delta h_\ell}^2.
\]
\end{proposition}

Thus each semantic step is the first-order image of the residual write, up to a controlled quadratic remainder.

For a fixed context $x$, let $e_y$ be a fixed unit direction in the chosen coordinates, either the target-token coordinate axis when included in the anchor set or a unit contrast direction. Define the \emph{signed target-direction contribution}
\[
a_\ell := \ip{\Delta z_\ell}{e_y}.
\]
The contribution profile $\{a_\ell\}$ sums to the net target-direction displacement,
\[
\sum_{\ell=0}^{L-1} a_\ell = \ip{z_L-z_0}{e_y},
\]
and records where target-direction displacement is accumulated, delayed, or reversed.

We use \emph{semantic imbalance} for trajectories in which later layers partially undo earlier progress before reaching the same endpoint. Three diagnostics make this precise:
\[
\text{motion energy } \sum_{\ell=0}^{L-1}\norm{\Delta z_\ell}^2,
\qquad
\text{curvature energy } \sum_{\ell=1}^{L-1}\norm{\Delta^2 z_\ell}^2,
\]
where $\Delta^2 z_\ell:=z_{\ell+1}-2z_\ell+z_{\ell-1}$ is the discrete second difference, together with the cumulative target-direction backtracking penalty
\[
\sum_{\ell=0}^{L-1}\bigl[-\ip{\Delta z_\ell}{e_y}\bigr]_+.
\]
The first tracks total semantic travel, the second tracks oscillation across depth, and the third isolates target-direction reversal. They flag detours, low-margin crossings, and oscillation between alternatives, without by themselves proving that a crossing is unnecessary.

\subsection{Cross-Frame Stability of Dynamics}
We next show that admissible frame changes preserve the coordinate geometry and step structure used below.

\begin{proposition}[Anchor-span coordinate equivalence]
\label{prop:coord_equivalence}
Let $E_{\mathcal{B}}\in\R^{d\times K}$ have full column rank and consider two restricted coordinate maps $\phi(h)=U_{\mathcal{B}}^\top h$ and $\tilde \phi(h)=\tilde U_{\mathcal{B}}^\top h$.
Assume both satisfy the restricted interface condition with errors $\varepsilon,\tilde\varepsilon<1$:
\[
\norm{U_{\mathcal{B}}^\top E_{\mathcal{B}}-I_K}\le \varepsilon,
\qquad
\norm{\tilde U_{\mathcal{B}}^\top E_{\mathcal{B}}-I_K}\le \tilde\varepsilon.
\]
Then on the anchor subspace $\mathrm{span}(E_{\mathcal{B}})$ the two coordinate systems are related by a near-identity linear transform. For any $h=E_{\mathcal{B}}c$,
\[
\tilde \phi(h)=\bigl(I_K + \Delta\bigr)\,\phi(h),
\qquad
\norm{\Delta}\le\frac{\varepsilon+\tilde\varepsilon}{1-\varepsilon}.
\]
\end{proposition}

\begin{corollary}[Relational stability]\label{cor:relational_stability}
Under the hypotheses of Proposition~\ref{prop:coord_equivalence}, let
\[
\eta:=\frac{\varepsilon+\tilde\varepsilon}{1-\varepsilon}.
\]
For any $u,v\in \mathrm{span}(E_{\mathcal B})$,
\[
\bigl|\ip{\tilde\phi(u)}{\tilde\phi(v)}-\ip{\phi(u)}{\phi(v)}\bigr|
\le (2\eta+\eta^2)\,\norm{\phi(u)}\,\norm{\phi(v)},
\]
and
\[
\norm{\tilde\phi(u)-\tilde\phi(v)}
\le (1+\eta)\,\norm{\phi(u)-\phi(v)}.
\]
If $\eta<1$, the matching lower bound also holds:
\[
(1-\eta)\,\norm{\phi(u)-\phi(v)}
\le \norm{\tilde\phi(u)-\tilde\phi(v)}.
\]
Thus pairwise similarities and distances are stable when the interface errors are small.
\end{corollary}

\begin{assumption}[Lipschitz stability]
\label{ass:lipschitz}
The coordinate map $\phi$ is $L_\phi$-Lipschitz. For all $h,h'\in\R^d$,
\[
\norm{\phi(h)-\phi(h')} \le L_\phi \norm{h-h'}.
\]
\end{assumption}

\begin{theorem}[Semantic-step stability bound]
\label{thm:stepbound}
Under \cref{ass:lipschitz}, for any residual update $\Delta h_\ell$,
\[
\norm{\Delta z_\ell}=\norm{\phi(h_\ell+\Delta h_\ell)-\phi(h_\ell)}\le L_\phi \norm{\Delta h_\ell}.
\]
In particular, if $\norm{\Delta h_\ell}\le U$ for all $\ell$, then $\norm{\Delta z_\ell}\le L_\phi U$ for all $\ell$.
\end{theorem}

\begin{theorem}[Frame-to-step stability transfer]
\label{thm:framestab}
Consider two linear-readout coordinate maps $\phi(h)=U_{\mathcal B}^\top h$ and $\tilde\phi(h)=\tilde U_{\mathcal B}^\top h$ built on the same anchor synthesis $E_{\mathcal{B}}\in\R^{d\times K}$, and assume both satisfy the restricted bi-invertibility condition (\cref{ass:interface}) with errors $\varepsilon,\tilde\varepsilon<1$. Define
\[
C_{\mathrm{frame}} := \frac{\varepsilon+\tilde\varepsilon}{\sigma_{\min}(E_{\mathcal{B}})}.
\]
Then on the anchor subspace $\mathrm{span}(E_{\mathcal{B}})$ the two coordinate systems are uniformly close, and for any $h=E_{\mathcal{B}}c$,
\[
\norm{\phi(h)-\tilde\phi(h)} \le C_{\mathrm{frame}}\norm{h}.
\]
Moreover, for any anchor-subspace update $\Delta h=E_{\mathcal{B}}\Delta c$,
\[
\norm{(\phi(h+\Delta h)-\phi(h))-(\tilde\phi(h+\Delta h)-\tilde\phi(h))}
\le C_{\mathrm{frame}}\norm{\Delta h}.
\]
\end{theorem}

Together, these results show that, under admissibility, nearby frames preserve the anchor-span geometry and induced step profiles used below. Continuous semantic steps can then be compared across depth without being dominated by readout drift. Voronoi assignments are reliable when their margins dominate frame distortion; for example, if $\|\tilde\phi(h)-\phi(h)\|_\infty\le \delta$ and $\operatorname{mar}(h)>2\delta$, then $\tilde\kappa(h)=\kappa(h)$.

\subsection{Endpoint-Matched Path Action}
We now introduce the comparison path. Endpoint matching makes pathwise deviation comparable. The semantic tube keeps the path on the trusted coarse Voronoi route up to margin relaxations, without copying the observed coordinates exactly.

\begin{proposition}[Contribution-profile telescoping]
\label{prop:contrib_tel}
Let $z^{\mathrm{obs}}_{0:L}$ be an observed teacher-forced trajectory and let $z^{\mathrm{cmp}}_{0:L}$ be any endpoint-matched comparison path. Assume
\[
z^{\mathrm{obs}}_0=z^{\mathrm{cmp}}_0,
\qquad
z^{\mathrm{obs}}_L=z^{\mathrm{cmp}}_L.
\]
Define
\[
g_\ell:=\ip{\Delta z^{\mathrm{obs}}_\ell-\Delta z^{\mathrm{cmp}}_\ell}{e_y}.
\]
Then
\[
\sum_{\ell=0}^{L-1} g_\ell = 0,
\qquad
\sum_{\ell=0}^{j-1} g_\ell = \ip{z^{\mathrm{obs}}_j-z^{\mathrm{cmp}}_j}{e_y},\quad j=1,\dots,L.
\]
\end{proposition}

Thus cumulative contribution-profile mismatch equals intermediate target-direction displacement between the observed path and any endpoint-matched comparison path.

\begin{definition}[Margin-relaxed semantic tube]
\label{def:tube}
Let $z^{\mathrm{obs}}_{0:L}$ be the observed trajectory and let $\tau_\ell:=\kappa(h^{\mathrm{obs}}_\ell)$ be its Voronoi index. Choose a trusted layer set $\mathcal I\subseteq\{1,\dots,L-1\}$, typically the layers whose margins dominate measurement error, and relaxation radii $\rho_\ell\ge 0$. The semantic tube induced by the observed Voronoi trace is
\[
\mathcal T_\rho(z^{\mathrm{obs}})
:=\left\{z_{0:L}:\begin{array}{l}
z_0=z^{\mathrm{obs}}_0,\quad z_L=z^{\mathrm{obs}}_L,\\
z_{\ell,\tau_\ell}\ge z_{\ell,j}-\rho_\ell\quad \forall \ell\in\mathcal I,\ \forall j
\end{array}\right\}.
\]
The constraints are linear dominance constraints in the semantic coordinates, so the tube is a closed convex polyhedron. For a realized trajectory and nonnegative radii, it contains $z^{\mathrm{obs}}$. If $\mathcal I=\varnothing$, the construction reduces to endpoint matching; inactive constraints leave the endpoint-only baseline unchanged.
\end{definition}

\begin{definition}[Path action]\label{def:action}
For a prediction event $(\theta,x,y)$, let $e_y\in\R^K$ be the fixed unit direction used for target-direction diagnostics. We define an action $S[z]$ on depth-discrete trajectories $z_{0:L}$ by
\begin{equation}
\label{eq:action}
S[z]
:= \alpha \sum_{\ell=0}^{L-1}\norm{\Delta z_\ell}^2
\;+\;
\beta \sum_{\ell=1}^{L-1}\norm{\Delta^2 z_\ell}^2
\;+\;
\gamma \sum_{\ell=0}^{L-1}\bigl[-\ip{\Delta z_\ell}{e_y}\bigr]_+,
\end{equation}
with $\alpha,\beta\ge 0$ not both zero, and $\gamma\ge 0$. The terms penalize excess motion, oscillation, and target-direction backtracking.
\end{definition}

If no target-direction diagnostic is used, set $\gamma=0$ and omit the backtracking term. The action scores feasible comparison paths rather than training the model; Section~\ref{sec:results} minimizes it inside the semantic tube.

\section{Traces and Compressibility}
\label{sec:results}
We now pass from descriptive path comparison to a selected baseline. Minimizing the path action from Section~\ref{sec:dynamics} over the margin-relaxed semantic tube produces the canonical trace. Unless stated otherwise, Euler--Lagrange statements concern the quadratic regime $\gamma=0$ on layers where tube inequalities are inactive.

\subsection{Minimum-Action Canonical Traces}
Under teacher forcing, the model induces an observed semantic trajectory $z^{\mathrm{TF}}_{0:L}$ with endpoints $z^{\mathrm{TF}}_0=z_0$ and $z^{\mathrm{TF}}_L=z_L$. With the semantic tube $\mathcal T_\rho(z^{\mathrm{TF}})$ from Definition~\ref{def:tube}, define
\begin{equation}
\label{eq:mintrace}
z^\star \in \arg\min_{z_{0:L}\in\mathcal T_\rho(z^{\mathrm{TF}})} S[z].
\end{equation}
We call $z^\star$ the \emph{canonical trace}: the diagnostic minimum-action comparison path that follows the trusted coarse Voronoi route up to the relaxation radii.

\begin{theorem}[Canonical-trace characterization]
\label{thm:spline}
Consider \eqref{eq:mintrace} with the action \eqref{eq:action}. The semantic tube from Definition~\ref{def:tube} is a nonempty closed convex polyhedron for realized trajectories and nonnegative radii. If $(\alpha,\beta)\neq(0,0)$ with $\alpha,\beta\ge 0$, the quadratic part is strictly convex on endpoint-constrained path variables. Hence the canonical trace is unique for $\gamma=0$ and remains unique after adding the convex backtracking term for $\gamma>0$. For $\gamma=0$, writing
\[
\Delta^2 z_\ell := z_{\ell+1}-2z_\ell+z_{\ell-1},
\qquad
\Delta^4 z_\ell := z_{\ell+2}-4z_{\ell+1}+6z_\ell-4z_{\ell-1}+z_{\ell-2},
\]
the coordinates satisfy, on every inactive interior layer,
\[
\beta\,\Delta^4 z^\star_\ell = \alpha\,\Delta^2 z^\star_\ell,
\qquad \ell=2,\dots,L-2.
\]
At active tube constraints, KKT multipliers for the corresponding linear dominance inequalities are added.
\end{theorem}

Here $\alpha$ prices total semantic travel and $\beta$ penalizes oscillatory correction. If the tube has no active constraints, the quadratic canonical trace is straight endpoint interpolation. Otherwise $z^\star$ is obtained from a sparse convex quadratic program with a banded Hessian and layer-local linear inequalities.

\subsection{Excess Action and Deviation Control}
Excess action is a scalar deviation diagnostic: the action gap dominates quadratic deviation energy, with strong-convexity modulus stated in Proposition~\ref{prop:strongconv_appendix}.

\begin{proposition}[Excess-action lower bound]
\label{prop:exact_gap}
Let $z^\star$ be a minimizer of \eqref{eq:mintrace} for the action \eqref{eq:action} with $\gamma\ge 0$, and let $z\in\mathcal T_\rho(z^{\mathrm{TF}})$ be any feasible trajectory. Writing $e_\ell:=z_\ell-z_\ell^\star$, one has
\[
S[z]-S[z^\star]
\ge \alpha\sum_{\ell=0}^{L-1}\norm{\Delta e_\ell}^2
+ \beta\sum_{\ell=1}^{L-1}\norm{\Delta^2 e_\ell}^2.
\]
When $\gamma=0$ and the active normal-cone pairing vanishes in the direction $z-z^\star$, the inequality is an equality.
\end{proposition}

\begin{corollary}[Deviation controls from excess action]
\label{cor:excess_controls}
In the setting of Proposition~\ref{prop:exact_gap}, define $\Delta S:=S[z]-S[z^\star]$. Then
\[
\alpha\sum_{\ell=0}^{L-1}\norm{\Delta z_\ell-\Delta z_\ell^\star}^2
\le \Delta S,
\qquad
\beta\sum_{\ell=1}^{L-1}\norm{\Delta^2 z_\ell-\Delta^2 z_\ell^\star}^2
\le \Delta S.
\]
If the endpoint-constrained problem has strong-convexity modulus $\mu>0$, then
\[
\sum_{\ell=1}^{L-1}\norm{z_\ell-z_\ell^\star}^2\le \frac{2\Delta S}{\mu}.
\]
For $z=z^{\mathrm{obs}}$, define $g_\ell:=\ip{\Delta z^{\mathrm{obs}}_\ell-\Delta z^\star_\ell}{e_y}$. If $\alpha>0$, then
\[
\max_{1\le j\le L}\left|\sum_{\ell=0}^{j-1} g_\ell\right|
\le \sum_{\ell=0}^{L-1}|g_\ell|
\le \sqrt{\frac{L}{\alpha}}\,\Delta S^{1/2}.
\]
Thus excess action controls step mismatch, curvature mismatch, interior displacement, and cumulative target-direction contribution-profile mismatch whenever the corresponding weights are positive.
\end{corollary}

\subsection{Compressibility and Trace Complexity}
The preceding bounds quantify deviation from the canonical trace. We now ask when such a trace admits a compact description: low curvature means that few semantic knots suffice to approximate the full depthwise trajectory.

\begin{definition}[Trajectory compressibility]
\label{def:compress}
A semantic trajectory $z_{0:L}\in(\R^K)^{L+1}$ is \emph{$(m,\varepsilon)$-compressible} if there exists a piecewise-linear trajectory $\tilde z_{0:L}$ with at most $m$ linear segments, whose breakpoints are the semantic knots, such that
\[
\frac{1}{L+1}\sum_{\ell=0}^L \norm{z_\ell-\tilde z_\ell}^2 \le \varepsilon^2.
\]
\end{definition}

\begin{definition}[Frame-relative trace complexity]
For a fixed admissible SemRF and tolerance $\varepsilon$, define $m_\varepsilon(z)$ as the smallest $m$ for which $z$ is $(m,\varepsilon)$-compressible. Up to knot locations, the trace needs $O(K\,m_\varepsilon(z))$ real values rather than $(L+1)K$. We use \emph{local knowledge density} only in this frame-relative sense: within the same frame and tolerance, smaller trace complexity means a more compact description of the same depthwise semantic transport.
\end{definition}

\paragraph{Connection to parameter efficiency.}
For a fixed architecture and dataset $\mathcal D$, parameters $\theta$ induce teacher-forced residual trajectories and hence SemRF trajectories $z_{0:L}(x;\theta)$. Interface alignment, including the pseudo-inverse tying case in Section~\ref{sec:setup}, stabilizes the semantic basis so that parameter settings can be compared through their induced semantic traces rather than through moving readout frames. Define the average trace complexity
\[
\bar m_\varepsilon(\theta;\mathcal D):=\frac{1}{|\mathcal D|}\sum_{x\in\mathcal D} m_\varepsilon\!\left(z^\star(x;\theta)\right),
\]
where $z^\star(x;\theta)$ is the canonical trace inside the tube induced by the trajectory of $\theta$ on $x$. For a fixed frame dimension, a conditional semantic-density proxy is
\[
\mathrm{KD}_{\varepsilon}(\theta;\mathcal D):=\frac{L+1}{\bar m_\varepsilon(\theta;\mathcal D)+1}.
\]
Equivalently, over a feasible family $\Theta_{\mathcal D}$ of admissible parameters satisfying the same data constraints, one may use
\[
\theta^\dagger\in\arg\min_{\theta\in\Theta_{\mathcal D}}
\frac{1}{|\mathcal D|}\sum_{x\in\mathcal D} S\!\left[z^\star(x;\theta)\right]
+\lambda\,\bar m_\varepsilon(\theta;\mathcal D)
\]
with $\lambda\ge0$ as a semantic-efficiency selector. Lower action and lower trace complexity then indicate higher local knowledge density and a more parameter-efficient semantic realization, relative to the parameter-to-trajectory map $\theta\mapsto z_{0:L}(x;\theta)$ and the fixed frame, data, and tolerance.

\begin{theorem}[Curvature-to-compressibility bound]
\label{thm:compress}
Let $z_{0:L}$ be any trajectory with curvature energy
\[
E_2(z):=\sum_{\ell=1}^{L-1}\norm{\Delta^2 z_\ell}^2.
\]
There exists a piecewise-linear $\tilde z_{0:L}$ with at most $m$ segments such that
\[
\frac{1}{L+1}\sum_{\ell=0}^L \norm{z_\ell-\tilde z_\ell}^2 \le \frac{c\,L^3\,E_2(z)}{m^4},
\]
where $c>0$ is a universal constant independent of $K$ and $L$. Hence $z$ is $(m,\varepsilon)$-compressible whenever
\[
m \ge \left(\frac{c\,L^3\,E_2(z)}{\varepsilon^2}\right)^{1/4}.
\]
\end{theorem}

For a quadratic canonical trace with $\gamma=0$ and $\beta>0$, one has $E_2(z^\star)\le S[z^\star]/\beta$. Theorem~\ref{thm:compress} therefore gives $(m,\varepsilon)$-compressibility whenever
\[
m\ge \left(\frac{c\,L^3\,S[z^\star]}{\beta\,\varepsilon^2}\right)^{1/4}.
\]
This justifies the local knowledge-density claim: once measurement, tolerance, and tube constraints are fixed, low-curvature admissible traces need fewer values than the full depth-by-coordinate representation. The Voronoi trace gives the coarse route, while the canonical trace and excess action quantify the remaining continuous motion.

\section{Conclusion}
\label{sec:conclusion}
SemRF isolates a local regime in which residual-stream semantics is well posed. On a validated anchor span, coordinates are stable, the same anchors induce a semantic Voronoi diagram, and admissible frame changes preserve the geometry needed for path comparison. The observed Voronoi sequence defines a margin-relaxed semantic tube; minimizing path action inside it gives a diagnostic canonical trace. With at least one positive quadratic weight, the trace is unique, satisfies a discrete spline equation away from active constraints, and reduces to straight interpolation only when no coarse-state constraint is active. Excess action controls the corresponding step, curvature, interior, and target-direction deviations.
Low curvature yields piecewise-linear compressibility, giving a local, frame- and tolerance-relative notion of knowledge density. Through the parameter-induced trajectory map, this connects to parameter efficiency by ranking lower-action and lower-complexity semantic transport within the same admissible frame and data constraints. These are local guarantees: they require controlled interface error, small projection residual, explicit tube constraints, and a fixed unit direction whenever target-direction diagnostics are used.


\bibliography{custom}

\appendix

\section{Additional Proofs}
\label{app:proofs}

The additional proofs to our research are as follows, grouped by the corresponding technical role.

\subsection{Coordinate Measurement and Interface Synchronization}

\begin{proof}\emph{of Proposition~\ref{prop:pitz_canonical}}.
Because $Z^\top Z=I_d$ and $T$ is invertible,
\[
W_{\mathrm U}^{+}=(ZT^\top)^+=T^{-\top}Z^\top.
\]
Using $W_{\mathrm{in}}=ZT^{-1}$ gives $W_{\mathrm{in}}^\top=T^{-\top}Z^\top$, so $W_{\mathrm U}^{+}=W_{\mathrm{in}}^\top$.
\end{proof}

\begin{proof}\emph{of Proposition~\ref{prop:approx_pit}}.
Let $E_{\mathcal B}^+=(E_{\mathcal B}^\top E_{\mathcal B})^{-1}E_{\mathcal B}^\top$.
Since $E_{\mathcal B}$ has full column rank, $E_{\mathcal B}^+E_{\mathcal B}=I_K$. Hence
\[
U_{\mathcal B}^\top E_{\mathcal B}-I_K=(U_{\mathcal B}^\top-E_{\mathcal B}^+)E_{\mathcal B},
\]
so if $\eta:=\|U_{\mathcal B}^\top-E_{\mathcal B}^+\|$, then
\[
\|U_{\mathcal B}^\top E_{\mathcal B}-I_K\|\le \eta\,\|E_{\mathcal B}\|.
\]
If instead $U_{\mathcal B}^\top=\tilde E_{\mathcal B}^+$ for a nearby full-column-rank synthesis matrix $\tilde E_{\mathcal B}$, then $\tilde E_{\mathcal B}^+\tilde E_{\mathcal B}=I_K$ and therefore
\[
U_{\mathcal B}^\top E_{\mathcal B}-I_K
=\tilde E_{\mathcal B}^+(E_{\mathcal B}-\tilde E_{\mathcal B}),
\]
which gives
\[
\|U_{\mathcal B}^\top E_{\mathcal B}-I_K\|\le \|\tilde E_{\mathcal B}^+\|\,\delta.
\]
\end{proof}

\begin{proof}\emph{of Proposition~\ref{prop:coord_distortion}}.
For $h_c=E_{\mathcal B}c$,
\[
\phi_{\mathrm{rd}}(h_c)-c=(U_{\mathcal B}^\top E_{\mathcal B}-I)c,
\]
so the first claim follows from \cref{ass:interface}. For the general decomposition $h=E_{\mathcal B}c+r$,
\[
\phi_{\mathrm{rd}}(h)-c=(U_{\mathcal B}^\top E_{\mathcal B}-I)c+U_{\mathcal B}^\top r,
\]
and the stated bound follows by the triangle inequality.
\end{proof}

\begin{proof}\emph{of Corollary~\ref{cor:trajectory_distortion}}.
The state bound is the second bound of Proposition~\ref{prop:coord_distortion}, applied at layer~$\ell$ with $h_\ell=E_{\mathcal{B}}c_\ell+r_\ell$. For the step bound, write
\[
\Delta z_\ell-\Delta c_\ell=(z_{\ell+1}-c_{\ell+1})-(z_\ell-c_\ell),
\]
and apply the state bound at layers $\ell+1$ and $\ell$.
\end{proof}

\subsection{Layerwise Dynamics and Frame Stability}

\begin{proof}\emph{of Proposition~\ref{prop:first_order}}.
Apply \cref{ass:linearize} with $h=h_\ell$ and $\delta=\Delta h_\ell$:
\[
\phi(h_\ell+\Delta h_\ell)-\phi(h_\ell)=J_\phi(h_\ell)\Delta h_\ell+r(h_\ell,\Delta h_\ell),
\]
and rename the remainder $\varepsilon_\ell$.
\end{proof}

\begin{proof}\emph{of Proposition~\ref{prop:coord_equivalence}}.
Let $A:=U_{\mathcal B}^\top E_{\mathcal B}$ and $\tilde A:=\tilde U_{\mathcal B}^\top E_{\mathcal B}$. Since $\|A-I\|\le \varepsilon<1$, the Neumann-series bound gives $A$ invertible with $\|A^{-1}\|\le 1/(1-\varepsilon)$. For $h=E_{\mathcal B}c$,
\[
\phi(h)=Ac,\qquad \tilde\phi(h)=\tilde A c=\tilde A A^{-1}\phi(h).
\]
Hence $\tilde\phi=(I+\Delta)\phi$ on the anchor span with $\Delta:=\tilde A A^{-1}-I=(\tilde A-A)A^{-1}$. Using
\[
\|\tilde A-A\|\le \|\tilde A-I\|+\|A-I\|\le \tilde\varepsilon+\varepsilon,
\]
one obtains
\[
\|\Delta\|\le \frac{\varepsilon+\tilde\varepsilon}{1-\varepsilon}.
\]
\end{proof}

\begin{proof}\emph{of Corollary~\ref{cor:relational_stability}}.
By Proposition~\ref{prop:coord_equivalence}, on the anchor span one has $\tilde\phi=(I+\Delta)\phi$ with
\[
\|\Delta\|\le \eta:=\frac{\varepsilon+\tilde\varepsilon}{1-\varepsilon}.
\]
For any $u,v\in \mathrm{span}(E_{\mathcal B})$,
\[
\ip{\tilde\phi(u)}{\tilde\phi(v)}-\ip{\phi(u)}{\phi(v)}
=\ip{\Delta\phi(u)}{\phi(v)}+\ip{\phi(u)}{\Delta\phi(v)}+\ip{\Delta\phi(u)}{\Delta\phi(v)}.
\]
Applying the Cauchy--Schwarz inequality and the operator-norm bound on $\Delta$ gives
\[
\bigl|\ip{\tilde\phi(u)}{\tilde\phi(v)}-\ip{\phi(u)}{\phi(v)}\bigr|
\le (2\eta+\eta^2)\,\norm{\phi(u)}\,\norm{\phi(v)}.
\]
For distances, set $w:=\phi(u)-\phi(v)$. Then
\[
\tilde\phi(u)-\tilde\phi(v)=(I+\Delta)w,
\]
so
\[
\norm{(I+\Delta)w}\le (1+\eta)\norm{w}.
\]
If $\eta<1$, the reverse triangle inequality also gives
\[
\norm{(I+\Delta)w}\ge \norm{w}-\norm{\Delta w}\ge (1-\eta)\norm{w}.
\]
This proves the stated distance bounds.
\end{proof}

\begin{proof}\emph{of Theorem~\ref{thm:stepbound}}.
The statement is immediate from \cref{ass:lipschitz}:
\[
\norm{\Delta z_\ell}=\norm{\phi(h_\ell+\Delta h_\ell)-\phi(h_\ell)}\le L_\phi\norm{\Delta h_\ell}.
\]
The uniform bound follows by substituting $\norm{\Delta h_\ell}\le U$.
\end{proof}

\begin{proof}\emph{of Theorem~\ref{thm:framestab}}.
For $h=E_{\mathcal B}c$,
\[
\phi(h)-\tilde\phi(h)=(U_{\mathcal B}^\top-\tilde U_{\mathcal B}^\top)E_{\mathcal B}c=((U_{\mathcal B}^\top E_{\mathcal B}-I)- (\tilde U_{\mathcal B}^\top E_{\mathcal B}-I))c.
\]
Hence
\[
\norm{\phi(h)-\tilde\phi(h)}\le (\varepsilon+\tilde\varepsilon)\norm{c}.
\]
Since $\|h\|=\|E_{\mathcal B}c\|\ge \sigma_{\min}(E_{\mathcal B})\|c\|$, it follows that
\[
\norm{\phi(h)-\tilde\phi(h)}\le \frac{\varepsilon+\tilde\varepsilon}{\sigma_{\min}(E_{\mathcal B})}\norm{h}=C_{\mathrm{frame}}\norm{h}.
\]
For the step bound, use linearity:
\[
(\phi(h+\Delta h)-\phi(h))-(\tilde\phi(h+\Delta h)-\tilde\phi(h))=(U_{\mathcal B}^\top-\tilde U_{\mathcal B}^\top)\Delta h.
\]
Writing $\Delta h=E_{\mathcal B}\Delta c$ and repeating the previous argument yields
\[
\norm{(\phi(h+\Delta h)-\phi(h))-(\tilde\phi(h+\Delta h)-\tilde\phi(h))}
\le \frac{\varepsilon+\tilde\varepsilon}{\sigma_{\min}(E_{\mathcal B})}\norm{\Delta h}=C_{\mathrm{frame}}\norm{\Delta h}.
\]
\end{proof}

\begin{proof}\emph{of Proposition~\ref{prop:contrib_tel}}.
Let $z^{\mathrm{cmp}}$ be any endpoint-matched comparison path, so $z^{\mathrm{obs}}_0=z^{\mathrm{cmp}}_0$ and $z^{\mathrm{obs}}_L=z^{\mathrm{cmp}}_L$. By definition,
\[
\sum_{\ell=0}^{L-1} g_\ell
= \sum_{\ell=0}^{L-1}\ip{\Delta z^{\mathrm{obs}}_\ell-\Delta z^{\mathrm{cmp}}_\ell}{e_y}
= \ip{z^{\mathrm{obs}}_L-z^{\mathrm{obs}}_0-(z^{\mathrm{cmp}}_L-z^{\mathrm{cmp}}_0)}{e_y}
=0.
\]
For every $j\in\{1,\dots,L\}$, the same telescoping computation gives
\[
\sum_{\ell=0}^{j-1} g_\ell
= \ip{z^{\mathrm{obs}}_j-z^{\mathrm{obs}}_0-(z^{\mathrm{cmp}}_j-z^{\mathrm{cmp}}_0)}{e_y}
= \ip{z^{\mathrm{obs}}_j-z^{\mathrm{cmp}}_j}{e_y},
\]
because the comparison path shares the same initial state as the observed path.
\end{proof}

\subsection{Trace Optimality and Compressibility}

\paragraph{Endpoint-constrained convexity.}
For fixed endpoints, collect the interior variables into $\mathbf z:=(z_1,\dots,z_{L-1})\in(\R^K)^{L-1}$. The quadratic part is
\[
Q[z]=\alpha\,\mathbf z^\top (D_1^\top D_1\otimes I_K)\mathbf z + \beta\,\mathbf z^\top (D_2^\top D_2\otimes I_K)\mathbf z + \text{boundary terms},
\]
where $D_1,D_2$ are endpoint-constrained difference matrices. Their nullspaces are trivial on zero-endpoint perturbations, so the quadratic Hessian is positive definite whenever $(\alpha,\beta)\neq(0,0)$; convex backtracking and tube constraints preserve uniqueness over the feasible set.

\begin{proposition}[Strong-convexity modulus]
\label{prop:strongconv_appendix}
In the endpoint-constrained regime, the quadratic part of the action is $\mu_0$-strongly convex in the interior variables, where
\[
\mu_0 := 8\alpha\sin^2\!\Bigl(\frac{\pi}{2L}\Bigr) + 32\beta\sin^4\!\Bigl(\frac{\pi}{2L}\Bigr).
\]
\end{proposition}

\begin{proof}\emph{of Proposition~\ref{prop:strongconv_appendix}}.
The Hessian is $H=2\alpha(D_1^\top D_1\otimes I_K)+2\beta(D_2^\top D_2\otimes I_K)$. The smallest eigenvalues of $D_1^\top D_1$ and $D_2^\top D_2$ are $4\sin^2(\pi/(2L))$ and $16\sin^4(\pi/(2L))$, respectively; summing the contributions gives the bound.
\end{proof}

\begin{proof}\emph{of Theorem~\ref{thm:spline}}.
With interior variables $\mathbf z=(z_1,\dots,z_{L-1})$, the tube is convex and, for $\gamma=0$, the objective has Hessian $H=2\alpha(D_1^\top D_1\otimes I_K)+2\beta(D_2^\top D_2\otimes I_K)$. By the convexity facts above, $H$ is positive definite whenever $(\alpha,\beta)\neq(0,0)$; adding the convex backtracking term preserves uniqueness over the convex tube. On inactive interior layers, varying $z^\star+\tau v$ with $v_0=v_L=0$ gives
\[
0=\alpha\sum_{\ell=0}^{L-1}\ip{\Delta z_\ell^\star}{\Delta v_\ell}
+\beta\sum_{\ell=1}^{L-1}\ip{\Delta^2 z_\ell^\star}{\Delta^2 v_\ell}.
\]
Discrete summation by parts yields the coefficient $-\alpha\Delta^2 z_\ell^\star+\beta\Delta^4 z_\ell^\star$, hence $\beta\Delta^4 z_\ell^\star=\alpha\Delta^2 z_\ell^\star$. Active Voronoi-face constraints add KKT normal-cone terms.
\end{proof}

\begin{proof}\emph{of Proposition~\ref{prop:exact_gap}}.
Write $S=Q+\gamma B$, where $Q$ is the quadratic part and $B$ the convex backtracking term. Let $C:=\mathcal T_\rho(z^{\mathrm{TF}})$ and $e:=z-z^\star$. Optimality gives $s\in\partial B(z^\star)$ and $n\in N_C(z^\star)$ such that
\[
\nabla Q(z^\star)+\gamma s+n=0.
\]
Since $\ip{n}{e}\le0$ for feasible $z\in C$, convexity of $B$ gives
\[
\gamma\bigl(B[z]-B[z^\star]\bigr)\ge \gamma\ip{s}{e}=-\ip{\nabla Q(z^\star)}{e}-\ip{n}{e}.
\]
Therefore
\[
S[z]-S[z^\star]
\ge Q[z]-Q[z^\star]-\ip{\nabla Q(z^\star)}{e}-\ip{n}{e}
\ge Q[z]-Q[z^\star]-\ip{\nabla Q(z^\star)}{e}.
\]
Expanding the quadratic part yields
\[
Q[z]-Q[z^\star]-\ip{\nabla Q(z^\star)}{e}
= \alpha\sum_{\ell=0}^{L-1}\norm{\Delta e_\ell}^2
+\beta\sum_{\ell=1}^{L-1}\norm{\Delta^2 e_\ell}^2.
\]
This proves the lower bound. When $\gamma=0$ and the active normal-cone pairing vanishes, the bound is tight.
\end{proof}

\begin{proof}\emph{of Corollary~\ref{cor:excess_controls}}.
The step and curvature bounds are immediate from Proposition~\ref{prop:exact_gap}. Strong convexity on endpoint-constrained variables gives
\[
S[\mathbf z] \ge S[\mathbf z^\star] + \frac{\mu}{2}\norm{\mathbf z-\mathbf z^\star}^2
= S[\mathbf z^\star]+\frac{\mu}{2}\sum_{\ell=1}^{L-1}\norm{z_\ell-z_\ell^\star}^2.
\]
For the contribution-profile bound, let $e_\ell:=z^{\mathrm{obs}}_\ell-z_\ell^\star$. Then $e_0=e_L=0$ and $g_\ell=\ip{\Delta e_\ell}{e_y}$. Since $\norm{e_y}=1$,
\[
\sum_{\ell=0}^{L-1}|g_\ell| \le \sum_{\ell=0}^{L-1}\norm{\Delta e_\ell}
\le \sqrt{L}\Bigl(\sum_{\ell=0}^{L-1}\norm{\Delta e_\ell}^2\Bigr)^{1/2}.
\]
Proposition~\ref{prop:exact_gap} bounds the final factor by $\alpha^{-1/2}\Delta S^{1/2}$, and every partial sum is bounded by the total variation.
\end{proof}

\begin{lemma}[Block interpolation bound]
\label{lem:block_interp}
Let $e_0,\dots,e_n\in\R^K$ satisfy $e_0=e_n=0$. Then
\[
\sum_{k=0}^{n}\norm{e_k}^2 \le 2n^4 \sum_{k=1}^{n-1}\norm{\Delta^2 e_k}^2.
\]
\end{lemma}

\begin{proof}\emph{of Lemma~\ref{lem:block_interp}}.
Write $g_k:=e_k-e_{k-1}$ for $k=1,\dots,n$. Since $e_0=e_n=0$, one has $\sum_{k=1}^n g_k=0$. Hence, for each $k$,
\[
g_k=\frac1n\sum_{i=1}^n (g_k-g_i).
\]
If $i<k$, then $g_k-g_i=\sum_{j=i}^{k-1}\Delta^2 e_j$, while if $i>k$, then $g_k-g_i=-\sum_{j=k}^{i-1}\Delta^2 e_j$. Therefore
\[
\norm{g_k}\le \sum_{j=1}^{n-1}\norm{\Delta^2 e_j}\le \sqrt{n-1}\Bigl(\sum_{j=1}^{n-1}\norm{\Delta^2 e_j}^2\Bigr)^{1/2}.
\]
Summing the first differences gives, for every $k$,
\[
\norm{e_k}=\Bigl\|\sum_{i=1}^{k} g_i\Bigr\|\le n^{3/2}\Bigl(\sum_{j=1}^{n-1}\norm{\Delta^2 e_j}^2\Bigr)^{1/2}.
\]
Squaring and summing over $k=0,\dots,n$ yields
\[
\sum_{k=0}^{n}\|e_k\|^2 \le (n+1)n^3 \sum_{j=1}^{n-1}\|\Delta^2 e_j\|^2 \le 2n^4 \sum_{j=1}^{n-1}\|\Delta^2 e_j\|^2,
\]
which proves the claim.
\end{proof}

\begin{proof}\emph{of Theorem~\ref{thm:compress}}.
Let the interval $\{0,\dots,L\}$ be partitioned into $m$ contiguous blocks $I_j=[s_j,s_{j+1}]$ with lengths $n_j:=s_{j+1}-s_j\le \lceil L/m\rceil$. Let $\tilde z$ be the piecewise-linear interpolant that agrees with $z$ at all block endpoints. On each block, define the interpolation error $e_\ell:=z_\ell-\tilde z_\ell$. Then $e_{s_j}=e_{s_{j+1}}=0$ and $\Delta^2 e_\ell=\Delta^2 z_\ell$ on interior indices of the block. Applying Lemma~\ref{lem:block_interp} on each block gives
\[
\sum_{\ell\in I_j}\|e_\ell\|^2 \le 2n_j^4 \sum_{\ell=s_j+1}^{s_{j+1}-1}\|\Delta^2 z_\ell\|^2.
\]
Summing over blocks and using $n_j\le \lceil L/m\rceil$ gives
\[
\sum_{\ell=0}^L\|z_\ell-\tilde z_\ell\|^2 \le 2\Bigl\lceil\frac{L}{m}\Bigr\rceil^4 E_2(z)
\le c_1\frac{L^4}{m^4}E_2(z)
\]
for a universal constant $c_1$. Dividing by $L+1$ yields the stated bound after absorbing fixed numerical factors into the universal constant $c$.
\end{proof}

\end{document}